\documentclass[11pt,a4paper]{article}
\usepackage[margin=1in]{geometry}
\usepackage{fullpage}
\usepackage{arxiv}
\usepackage[compress,sort]{natbib}
\usepackage{caption}
\usepackage[utf8]{inputenc} 
\usepackage[T1]{fontenc}    
\usepackage{hyperref}       
\usepackage{url}            
\usepackage{booktabs}       
\usepackage{amsfonts}       
\usepackage{nicefrac}       
\usepackage{microtype}      
\usepackage{amssymb}
\usepackage{graphicx}
\usepackage{color}
\usepackage{subcaption}
\usepackage{amsmath}
\usepackage{amsthm}
\usepackage{mathrsfs}
\usepackage{tabularx}
\usepackage{fancyvrb}
\usepackage{url}
\usepackage{siunitx}
\usepackage{comment}
\usepackage{algorithm}
\usepackage{algorithmic}
\usepackage{wrapfig}
\usepackage{mathtools}
\usepackage{thmtools,thm-restate}
\usepackage{paralist} 
\usepackage{multirow}
\usepackage{scrextend}
\usepackage{enumitem}
\usepackage{quoting}
\usepackage{hyperref}
\usepackage{thm-restate}
\usepackage{tabularx}
\usepackage{xcolor}
\usepackage{makecell}
\usepackage{adjustbox}
\usepackage{mathpazo}
\usepackage{times}

\theoremstyle{definition}
\newtheorem{definition}{Definition}[section]

\newtheorem{example}{Example}[section]
\newtheorem{theorem}{Theorem}[section]
\newtheorem{lemma}{Lemma}[section]
\newtheorem{proposition}{Proposition}[section]

\newtheorem{assumption}{Assumption}[section]

\newcommand{\ie}{i.e.}
\newcommand{\eg}{e.g.}

\newcommand{\dist}{\mu}

\newcommand{\Hzo}{H_{0\mbox{-}1}}
\newcommand{\xxspace}{\mathcal{X}}
\newcommand{\yyspace}{\mathcal{Y}}
\newcommand{\zzspace}{\mathcal{Z}}

\newcommand{\Ypred}{\widehat{Y}}

\newcommand{\RR}{\mathbb{R}}

\newcommand{\Exp}{\mathbb{E}}
\newcommand{\HH}{\mathcal{H}}

\newcommand{\xx}{\mathbf{x}}

\DeclareMathOperator*{\argmin}{arg\,min}

\newcommand{\tr}{\text{Tr}}
\newcommand{\defeq}{\vcentcolon=}

\newcommand{\eps}{\varepsilon}

\DeclarePairedDelimiterX{\inp}[2]{\langle}{\rangle}{#1, #2}

\title{\huge \textbf{Costs and Benefits of Fair Regression}}

\author{%
  Han Zhao \\
  Department of Computer Science\\
  University of Illinois at Urbana-Champaign\\
  \texttt{hanzhao@illinois.edu} \\
}
\date{}
\begin{document}

\maketitle

\begin{abstract}
    Real-world applications of machine learning tools in high-stakes domains are often regulated to be fair, in the sense that the predicted target should satisfy some quantitative notion of parity with respect to a protected attribute. However, the exact tradeoff between fairness and accuracy with a real-valued target is not entirely clear. In this paper, we characterize the inherent tradeoff between statistical parity and accuracy in the regression setting by providing a lower bound on the error of any fair regressor. Our lower bound is sharp, algorithm-independent, and admits a simple interpretation: when the moments of the target differ between groups, any fair algorithm has to make an error on at least one of the groups. We further extend this result to give a lower bound on the joint error of any (approximately) fair algorithm, using the Wasserstein distance to measure the quality of the approximation. With our novel lower bound, we also show that the price paid by a fair regressor that does not take the protected attribute as input is less than that of a fair regressor with explicit access to the protected attribute. On the upside, we establish the first connection between individual fairness, accuracy parity, and the Wasserstein distance by showing that if a regressor is individually fair, it also approximately verifies the accuracy parity, where the gap is given by the Wasserstein distance between the two groups. Inspired by our theoretical results, we develop a practical algorithm for fair regression through the lens of representation learning, and conduct experiments on a real-world dataset to corroborate our findings.
\end{abstract}

\section{Introduction}
High-stakes domains, \eg, loan approvals, and credit scoring, have been using machine learning tools to help make decisions. A central question in these applications is whether the algorithm makes fair decisions, in the sense that certain sensitive data does not influence the outcomes or accuracy of the learning algorithms. For example, as regulated by the General Data Protection Regulation (GDPR, Article 22 Paragraph 4)~\citep{gdpr2021}, ``decisions which produces legal effects concerning him or her or of similar importance shall not be based on certain personal data'', including race, religious belief, etc. As a result, using the sensitive data directly in algorithm is often prohibited. However, due to the redundant encoding, redlining, and other problems, this ``fairness through blindness'' is often not sufficient to ensure algorithmic fairness in automated decision-making processes. 

Many works have produced methods aiming at reducing unfairness~\citep{calmon2017optimized,chi2021understanding,hardt2016equality,agarwal2019fair,feldman2015certifying,beutel2017data,lum2016statistical} under various contexts. However, the question of the price that we need to pay for enforcing various fairness definitions in terms of the accuracy of these tools is less explored. In this paper, we attempt to answer this question by characterizing a tradeoff between statistical parity and accuracy in the regression setting, where the regressor is prohibited to use the sensitive attribute directly. Among many definitions of fairness~\citep{verma2018fairness} in the literature, statistical parity asks the predictor to be statistically independent of a predefined protected attribute, \eg, race, gender, etc. While empirically it has long been observed that there is an underlying tension between accuracy and statistical parity~\citep{calders2013controlling,zliobaite2015relation,berk2017convex,agarwal2019fair} in both classification and regression settings, theoretical understanding of this tradeoff in regression is limited. In the case of classification, \citet{menon2018cost} explored such tradeoff in terms of the fairness frontier function under the context of cost-sensitive binary classification. \citet{zhao2019inherent} provided a characterization of such tradeoff in binary classification. Recently, \citet{chzhen2020fair} and \citet{le2020projection} concurrently derived an analytic bound to characterize the price of statistical parity in regression using Wasserstein barycentres when the learner can take the sensitive attribute explicitly as an input. 

In this paper, we derive the first lower bound to characterize the inherent tradeoff between fairness and accuracy in the regression setting under general $\ell_p$ loss when the regressor is prohibited to use the sensitive attribute directly during the inference stage. Our main theorem can be informally summarized as follows:
\begin{quote}
    \itshape
    For \emph{any} fair algorithm satisfying statistical parity, it has to incur a large error on at least one of the demographic subgroups when the moments of the target variable differ across groups. Furthermore, if the population of the two demographic subgroups are imbalanced, the minorities could still suffer from the reduction in accuracy even if the global accuracy does not seem to reduce.
\end{quote}
We emphasize that the above result holds in the noiseless setting as well, where there exists (unfair) algorithms that are perfect on both demographic subgroups. Hence it highlights the inherent tradeoff due to the coupling between statistical parity and accuracy in general, not due to the noninformativeness of the input. We also extend this result to the general noisy setting when only apporoximate fairness is required. Our bounds are algorithm-independent, and do not make any distributional assumptions. To illustrate the tightness of the lower bound, we also construct a problem instance where the lower bound is attained. In particular, it is easy to see that in an extreme case where the group membership coincides with the target task, a call for exact statistical parity will inevitably remove the perfect predictor. At the core of our proof technique is the use of the Wasserstein metric and its contraction property under certain Lipschitz assumption on the regression predictors. 

On the positive side, we establish the first connection between individual fairness~\citep{dwork2012fairness}, a more fine-grained notion of fairness, and accuracy parity. Roughly speaking, an algorithm is said to be individually fair if it treats similar individuals similarly. We show that if a regressor is individually fair, then it also approximately verifies the accuracy parity. Interestingly, the gap in this approximation is exactly given by the Wasserstein distance between the distributions across groups. 


Although our main focus is to understand the costs and benefits of using Wasserstein regularization for fair regression, our analysis using also naturally suggests a practical algorithm to achieve statistical parity and accuracy parity simultaneously in regression by learning fair representations. The idea is relatively simple and intuitive: it suffices if we can ensure that the representations upon which the regressor applies are approximately fair (measured by Wasserstein distance). Finally, we also conduct experiments on a real-world dataset to corroborate our theoretical findings. 


\section{Preliminaries}
\label{sec:preliminary}
\paragraph{Notation}
We consider a general regression setting where there is a joint distribution $\dist$ over the triplet $T = (X, A, Y)$, where $X\in\xxspace\subseteq\RR^d$ is the input vector, $A\in\{0, 1\}$\footnote{Our main results could be extended to the case where $A$ can take finitely many values.} is the protected attribute, e.g., race, gender, etc., and $Y\in\yyspace\subseteq [-1, 1]$ is the target output. Lower case letters $\xx$, $a$ and $y$ are used to denote the instantiation of $X$, $A$ and $Y$, respectively. Let $\HH$ be a hypothesis class of predictors from input to output space. Throughout the paper, we focus on the setting where the regressor \emph{cannot} directly use the sensitive attribute $A$ to form its prediction. However, note that even if the regressor does not explicitly take the protected attribute $A$ as input, this \emph{fairness through blindness} mechanism can still be biased due to the redundant encoding issue~\citep{barocas2017fairness}. To keep the notation uncluttered, for $a\in\{0, 1\}$, we use $\dist_a$ to mean the conditional distribution of $\dist$ given $A = a$. The zero-one entropy of $A$~\citep[Section 3.5.3]{grunwald2004game} is denoted as $\Hzo(A)\defeq 1 - \max_{a\in\{0, 1\}}\Pr(A = a)$. Furthermore, we use $F_{\dist}$ to represent the cumulative distribution function of $\dist$, i.e., for $z\in\RR$, $F_{\dist}(z)\defeq\Pr_\dist((-\infty, z])$. In this paper, we assume that the density of $\dist_i$ and its corresponding pushforward under proper transformation (w.r.t.\ the Lebesgue measure $\lambda$) is universally bounded above, \ie, $\|d\dist_i / d\lambda\|_\infty \leq C$, $\forall i\in\{0, 1\}$. Given a feature transformation function $g:\xxspace\to\zzspace$ that maps instances from the input space $\xxspace$ to feature space $\zzspace$, we define $g_\sharp\dist\defeq \dist\circ g^{-1}$ to be the induced distribution (pushforward) of $\dist$ under $g$, \ie, for any measurable event $E'\subseteq\zzspace$, $\Pr_{g_\sharp\dist}(E') \defeq \Pr_\dist(g^{-1}(E')) = \Pr_\dist(\{x\in\xxspace\mid g(x)\in E'\})$. We use $\dist_Y$ to denote the marginal distribution of $Y$ from a joint distribution $\dist$ over $Y$ and some other random variables. With slight abuse of notation, occasionally we also use $Y_\sharp\dist$ to denote the marginal distribution of $Y$ from the joint distribution $\dist$, \ie, projection of $\dist$ onto the $Y$ coordinate. Throughout the paper, we make the following assumption that the probability density of any continuous random variable to be bounded: 
\begin{assumption}
\label{assumption:bounded}
There exists a constant $C$ such that the density of $\dist'$ (w.r.t.\ the Lebesgue measure $\lambda$) is universally bounded above, i.e., $\|d\dist' / d\lambda\|_\infty \leq C$.
\end{assumption}

\paragraph{Fairness Definition}
We mainly focus on group fairness where the group membership is given by the protected attribute $A$. In particular, \emph{statistical parity} asks that the predictor should be statistically independent of the protected attribute. In binary classification, this requirement corresponds to the notion of equality of outcome~\citep{holzer2006affirmative}, and it says that the outcome rate should be equal across groups. 
\begin{definition}[Statistical Parity]
    \label{def:sp}
    Given a joint distribution $\dist$, a classifier $\Ypred = h(X)$, satisfies \emph{statistical parity} if $\Ypred$ is independent of $A$.
\end{definition}
Since $\Ypred$ is continuous, the above definition implies that $\Pr_{\dist_0}(\Ypred\in E) = \Pr_{\dist_1}(\Ypred\in E)$ for any measurable event $E\subseteq\RR$. Statistical parity has been adopted as definition of fairness in a series of work~\citep{calders2009building,edwards2015censoring,johndrow2019algorithm,kamiran2009classifying,kamishima2011fairness,louizos2015variational,zemel2013learning,madras2018learning}. 

\paragraph{Fair Regression}
Given a joint distribution $\dist$, the $\ell_p$ error of a predictor $\Ypred = h(X)$ under $\dist$ for $p\geq 1$ is defined as
\begin{equation}
    \varepsilon_{p,\dist}(\Ypred)\defeq \left(\Exp_{\dist}\left[|\Ypred - Y|^p\right]\right)^{1/p}.
    \label{equ:lp}
\end{equation}
As two notable special cases, when $p=2$, the above definition reduces to the square root of the usual mean-squared-error (MSE); when $p=1$,~\eqref{equ:lp} becomes the mean-absolute-error (MAE) of the predictor. To make the notation more compact, we may drop the subscript $\dist$ when it is clear from the context. The main departure from prior works on classification is that both $Y$ and $\Ypred (h(X))$ are allowed to be real-valued rather than just categorical. Under statistical parity, the problem of fair regression~\citep{agarwal2019fair} can be understood as the following constrained optimization problem:
\begin{equation}
    \label{equ:opt}
    \begin{aligned}
         & \underset{h\in\HH}{\text{minimize}} &  & \Exp_{\dist}\left[|h(X) - Y|^p\right]                                                            \\
         & \text{subject to}                   &  & \left|\Pr_{\dist_0}(h(X)\leq z) - \Pr_{\dist_1}(h(X)\leq z)\right|\leq\epsilon,~\forall z\in\RR.
    \end{aligned}
\end{equation}
Note that since $\Ypred = h(X)\in\RR$ is a real-valued random variable and $A$ is binary, the constraint in the above optimization formulation asks that the conditional cumulative distributions of $\Ypred$ are approximately equal across groups, which is an additive approximation to the original definition of statistical parity. Formally, the constraint in~\eqref{equ:opt} is known as the Kolmogorov-Smirnov distance:
\begin{definition}[Kolmogorov-Smirnov distance]
    For two probability distributions $\dist$ and $\dist'$ over $\RR$, the \emph{Kolmogorov-Smirnov distance} $K(\dist, \dist')$ is $K(\dist, \dist')\defeq\sup_{z\in\RR} |F_\dist(z) - F_{\dist'}(z)|$.
\end{definition}
With the Kolmogorov-Smirnov distance, we can define the $\epsilon$-statistical parity for a regressor $h$:
\begin{definition}[$\epsilon$-Statistical Parity]
    \label{def:esp}
    Given a joint distribution $\dist$ and $0 \leq \epsilon \leq 1$, a regressor $\Ypred = h(X)$, satisfies $\epsilon$-\emph{statistical parity} if $K(h_\sharp\dist_0, h_\sharp\dist_1)\leq\epsilon$.
\end{definition}
Clearly, the slack variable $\epsilon$ controls the quality of approximation and when $\epsilon = 0$ it reduces to asking exact statistical parity as defined in Definition~\ref{def:sp}. 

\paragraph{Wasserstein Distance}
Given two random variables $T$ and $T'$ with the corresponding distributions $\dist$ and $\dist'$, let $\Gamma(\dist, \dist')$ denote the set of all couplings $\gamma$ of $\dist$ and $\dist'$, i.e., $\gamma_T = \dist$ and $\gamma_{T'} = \dist'$. The \emph{Wasserstein distance} between the pair of distributions $\dist$ and $\dist'$ is defined as follows:
\begin{equation}
    W_p(\dist, \dist')\defeq \left(\inf_{\gamma\in\Gamma(\dist, \dist')}\int \|T - T'\|^p~d\gamma\right)^{1/p},
\end{equation}
where $p \geq 1$ and throughout this paper we fix $\|\cdot\|$ to be the $\ell_2$ norm. For the special case where both $\dist$ and $\dist'$ are distributions over $\RR$, the Wasserstein distance $W_p(\dist,\dist')$ admits the following equivalent characterization~\citep{kolouri2017optimal}:
\begin{equation}
    W_p(\dist, \dist') = \left(\int_0^1 |F_{\dist}^{-1}(t) - F_{\dist'}^{-1}(t)|^p~dt\right)^{1/p},
    \label{equ:quantile}
\end{equation}
where $F_{\dist}^{-1}(t)$ denotes the generalized inverse of the cumulative distribution function, i.e., $F_\dist^{-1}(t) = \inf_{z\in\RR}\{z: F(z) \geq t\}$. The above closed form formulation will be particularly useful in our later analysis. When $p = 1$, the Wasserstein distance is also called the \emph{Earth Mover distance}, and it admits a dual representation in a variational form using $\sup$ rather than $\inf$: $W_1(\dist, \dist') = \sup_{f:\|f\|_L\leq 1}\left|\int f~d\dist - \int f~d\dist'\right|$, where $\|f\|_L\defeq \sup_{\xx\neq\xx'}|f(\xx) - f(\xx')|/|\xx - \xx'|$ is the Lipschitz seminorm of $f$. It is well-known that convergences of measures under the Wasserstein distance implies weak convergence, i.e., convergence in distribution~\citep{gibbs2002choosing}. Furthermore, compared with other distance metrics including total variation (TV), Jensen-Shannon distance, etc.\ that ignore the geometric structure of the underlying space, Wasserstein distance often allows for more robust applications, e.g., the Wasserstein GAN~\citep{arjovsky2017wasserstein}, domain adaptation~\citep{courty2017joint}, etc., due to the its Lipschitz continuous constraint in the dual representation. Moreover, unlike the KL divergence, the Wasserstein distance between two measures is generally finite even when neither measure is absolutely continuous with respect to the other, a situation that often arises when considering empirical distributions arising in practice. Furthermore, unlike the TV-distance, the Wasserstein distance inherently depends on the geometry of the underlying space, whereas the TV distance is invariant under any bijective mapping. 

\section{Main Results}
\label{sec:main}
Recently, \citet{agarwal2019fair} proposed a reduction-based approach to tackle~\eqref{equ:opt} by solving a sequence of cost-sensitive problems. By varying the slack variable $\epsilon$, the authors also empirically verified the unavoidable tradeoff between statistical parity and accuracy in practice. However, to the best of our knowledge, a quantitative characterization on the exact tradeoff between fairness and accuracy is still missing. In this section, we seek to answer the following intriguing and important question:
\begin{quote}
    \itshape
    In the setting of regression, what is the minimum error that any fair algorithm has to incur, and how does this error depend on the coupling between the target and the protected attribute?
\end{quote}
In what follows we shall first provide a simple example to illustrate this tradeoff. This example will give readers a flavor the kind of impossibility result we are interested in proving. We then proceed to formally present our first theorem which exactly answers the above question, even if only approximate fairness is satisfied. We conclude this section by some discussions on the implications of our results.

\paragraph{A Simple Example}
As a warm-up, let us consider an example to showcase the potential tradeoff between statistical parity and accuracy. But before our construction, it should be noted that the error $\eps_{p,\dist}(\Ypred)$ bears an intrinsic lower bound for any deterministic predictor $\Ypred = h(X)$, i.e., the noise in the underlying data distribution $\dist$. Hence to simplify our discussions, in this example we shall construct distributions such that there is no noise in the data, i.e., for $a\in\{0, 1\}$, there exists a ground-truth labeling function $h_a^*$ such that $Y = h_a^*(X)$ on $\dist_a$. Realize that such simplification will only make it harder for us to prove lower bound on $\eps_{p,\dist_a}$ since there exists predictors that are perfect.
\begin{example}[Target coincides with the protected attribute]
    \label{exp:simple}
    For $a\in\{0, 1\}$, let the marginal distribution $X_\sharp\dist_a$ be a uniform distribution over $\{0, 1\}$. Let $Y = a$ be a constant. Hence by construction, on the joint distribution, we have $Y = A$ hold. Now for any fair predictor $\Ypred = h(X)$, the statistical parity asks $\Ypred$ to be independent of $A$.
    However, no matter what value $h(x)$ takes, we always have $|h(x)| + |h(x) - 1| \geq 1$. Hence for any predictor $h:\xxspace\to\RR$:
    \begin{equation*}
        \eps_{1,\dist_0}(h) + \eps_{1, \dist_1}(h) = \frac{1}{2}|h(0) - 0| + \frac{1}{2}|h(1) - 0| + \frac{1}{2}|h(0) - 1| + \frac{1}{2}|h(1) - 1| \geq \frac{1}{2} + \frac{1}{2} = 1.
    \end{equation*}
    This shows that for any fair predictor $h$, the sum of $\ell_1$ errors of $h$ on both groups has to be at least 1. 
    On the other hand, there exists a trivial unfair algorithm that makes no error on both groups by also taking the protected attribute into consideration: $\forall x\in\{0, 1\}, h^*(x) = 0$ if $A = 0$ else $h^*(x) = 1$.
\end{example}

\subsection{The Cost of Statistical Parity under Noiseless Setting}
\label{sec:first}
The example in the previous section corresponds to a worst case where $Y = A$. On the other hand, it is also clear that when the target variable $Y$ is indeed independent of the protected attribute $A$, there will be no tension between statistical parity and accuracy. The following theorem exactly characterizes the tradeoff between fairness and accuracy by taking advantage of the relationship between $Y$ and $A$:
\begin{restatable}{theorem}{exactlowerbound}
    \label{thm:exact}
    Let $\Ypred = h(X)$ be a predictor. If $\Ypred$ satisfies statistical parity, then $\forall p\geq 1$,
    \begin{equation}
        \eps_{p,\dist_0}(\Ypred) + \eps_{p,\dist_1}(\Ypred) \geq W_p(Y_\sharp\dist_0, Y_\sharp\dist_1).
        \label{equ:lowerbound}
    \end{equation}
\end{restatable}
We provide a proof by picture to illustrate the high-level idea of the proof in Fig.~\ref{fig:proof}. For the special case of $p = 1$ and $p = 2$, Theorem~\ref{thm:exact} gives the following lower bounds on the sum of MAE and MSE on both groups respectively:
\begin{restatable}{corollary}{corexact}
    If $\Ypred$ satisfies statistical parity, then $\eps_{1,\dist_0}(\Ypred) + \eps_{1,\dist_1}(\Ypred) \geq |\Exp_{\dist_0}[Y] - \Exp_{\dist_1}[Y]|$ and $\eps^2_{2,\dist_0}(\Ypred) + \eps^2_{2,\dist_1}(\Ypred) \geq \frac{1}{2}|\Exp_{\dist_0}[Y] - \Exp_{\dist_1}[Y]|^2$.
    \label{cor:exact}
\end{restatable}
\paragraph{Remark}
First of all, the lower bound $W_p(Y_\sharp\dist_0, Y_\sharp\dist_1)$ corresponds to a measure of the distance between the marginal distributions of $Y$ conditioned on $A = 0$ and $A = 1$ respectively. Hence when $A$ is independent of $Y$, we will have $Y_\sharp\dist_0 = Y_\sharp\dist_1$ so that the lower bound gracefully reduces to 0, \ie, no essential tradeoff between fairness and accuracy. On the other extreme, consider $Y = cA$, where $c > 0$. In this case $A$ fully describes $Y$ and it is easy to verify that $W_p(Y_\sharp\dist_0, Y_\sharp\dist_1) = c$, which means the lower bound also takes into account the magnitude of the target variable $Y$. For a protected attribute $A$ that takes more than 2 values, we could extend Theorem~\ref{thm:exact} by considering all possible pairwise lower bounds and average over them. Furthermore, the lower bound is sharp, in the sense that there exists problem instances that achieve the above lower bound, \eg, Example~\ref{exp:simple}. As another example, consider the following Gaussian case:
\begin{example}[Gaussian case]
    \label{exp:gaussian}
    For $a\in\{0, 1\}$, let the marginal distribution $X_\sharp\dist_a$ be a standard Gaussian distribution $\mathcal{N}(0, I_d)$ and assume $A\perp X$. Fix $w\in\RR^d$ with $\|w\| = 1$, and construct $Y_0 = w^T X - 1$ and $Y_1 = w^TX + 1$. Now for any regressor $\Ypred = h(X)$, due to the data-processing inequality, $\Ypred \perp A$ so $\Ypred$ is fair. However, consider the $\ell_2$ error of $h$ on both groups:
    \begin{align*}
        \eps_{2,\dist_0}(h) + \eps_{2, \dist_1}(h) &= \Exp_{X}^{1/2}[(h(X) - Y_0)^2] + \Exp_X^{1/2}[(h(X) - Y_1)^2] \\
        &\geq \Exp_{X}[|h(X) - Y_0|] + \Exp_X[|h(X) - Y_1|] \geq \Exp_{X}[|Y_0 - Y_1|] = 2.
    \end{align*}
    On the other hand, note that the distributions of $Y_0$ and $Y_1$ are $\mathcal{N}(-1, 1)$ and $\mathcal{N}(1, 1)$, respectively. The analytic formula~\citep[Proposition 7]{givens1984class} for the $W_2$ distance between two Gaussians $\mathcal{N}(m_0, \Sigma_0)$ and $\mathcal{N}(m_1, \Sigma_1)$ is
    \begin{equation*}
        W^2_2(\mathcal{N}(m_0, \Sigma_0), \mathcal{N}(m_1, \Sigma_1)) = \|m_0 - m_1\|^2 + \tr\left(\Sigma_0 + \Sigma_1 - 2\left(\Sigma_0^{1/2}\Sigma_1\Sigma_0^{1/2}\right)^{1/2}\right),
    \end{equation*}
    which shows that $W_2(Y_0, Y_1) = |-1 - 1| = 2$. Further, consider $\Ypred^* = h^*(X) = w^TX$, then 
    \begin{equation*}
     \eps_{2,\dist_0}(h^*) + \eps_{2, \dist_1}(h^*) = \Exp_{X}^{1/2}[(h^*(X) - Y_0)^2] + \Exp_X^{1/2}[(h^*(X) - Y_1)^2] = 1 + 1 = 2.    
    \end{equation*}
    Hence $h^*$ achieves the lower bound and the lower bound is verified. 
\end{example}

It is worth pointing out that the lower bound in Theorem~\ref{thm:exact} is algorithm-independent and it holds on the population distribution. That being said, by using recent tail bounds~\citep{lei2020convergence,weed2019sharp} on the expected Wasserstein distance between the empirical distributions and its population counterpart, it is not hard to extend Theorem~\ref{thm:exact} to obtain a finite sample high probability bound of Theorem~\ref{thm:exact}:
\begin{restatable}{theorem}{finite}
\label{thm:finite}
    Let $\Ypred = h(X)$ be the predictor and $\hat{\dist}$ be an empirical distribution induced from a sample of size $n$ drawn from $\dist$. If $\Ypred$ satisfies statistical parity, then there exists an absolute constant $c_1 > 0$ such that for $0 < \delta < 1$, with probability at least $1-\delta$ over the draw of the sample, 
    \begin{equation}
    \small
        \eps_{2,\dist_0}(\Ypred) + \eps_{2,\dist_1}(\Ypred) \geq \eps_{1,\dist_0}(\Ypred) + \eps_{1,\dist_1}(\Ypred) \geq W_1(Y_\sharp\hat{\dist}_0, Y_\sharp\hat{\dist}_1) - \left(2c_1 + \sqrt{2\log(2/\delta)}\right)\sqrt{\frac{1}{n}}.
        \label{equ:fl1}
    \end{equation}
\end{restatable}
\paragraph{Remark}
It is possible to obtain better lower bounds for the $\ell_2$ error in Theorem~\ref{thm:finite}, but that requires making more assumptions on the underlying distribution $\dist$, \eg, strongly log-concave density. The first term in the lower bound, $W_1(Y_\sharp\hat{\dist}_0, Y_\sharp\hat{\dist}_1)$, could be efficiently estimated from the data by solving a linear program~\citep[Problem (3)]{cuturi2014fast}. Furthermore, it is worth pointing out that the lower bound in Theorem~\ref{thm:finite} applies to all the predictors $\Ypred$ and is insensitive to the marginal distribution of $A$. As a comparison, let $\alpha\defeq \Pr_\dist(A = 0)$, then $\eps_{p,\dist}(\Ypred) = \alpha \eps_{p,\dist_0}(\Ypred) + (1-\alpha)\eps_{p,\dist_1}(\Ypred)$. In this case if the group ratio is imbalanced, the overall error $\eps_{p,\dist}(\Ypred)$ could still be small even if the minority group suffers a large error. Using Theorem~\ref{thm:exact}, we can also bound the joint error over all the population:
\begin{figure}[tb]
    \centering
    \includegraphics[width=0.6\linewidth]{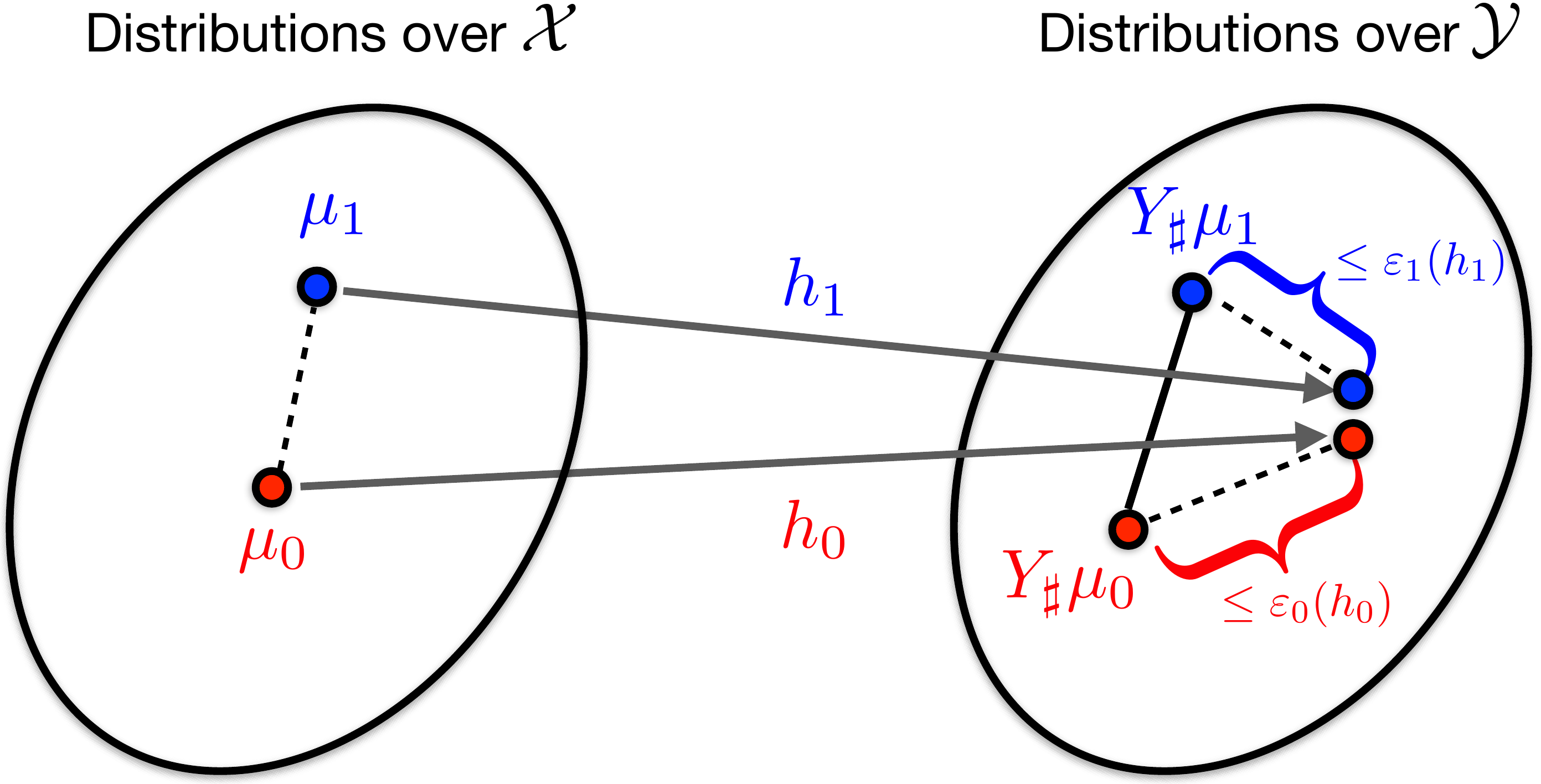}
    \caption{Proof by picture. The predictors on $\dist_0$ and $\dist_1$ induce two predictive distributions over $\yyspace$. Applying a chain of triangle inequalities (with $W_p(\cdot,\cdot)$) to the quadrilateral in the right circle and using the fact that the Wasserstein distance is a lower bound of the regression error then completes the proof. Note that here the predictors $h_0$ and $h_1$ over $\dist_0$ and $\dist_1$ need not to be the same.}
    \label{fig:proof}
\end{figure}
\begin{restatable}{corollary}{jointlowerbound}
    \label{cor:joint}
    Let $\Ypred = h(X)$ be a predictor. If $\Ypred$ satisfies statistical parity, then $\forall p\geq 1$, the joint error has the following lower bound:
    \begin{equation}
        \eps_{p,\dist}(\Ypred) \geq \Hzo(A)\cdot W_p(Y_\sharp\dist_0, Y_\sharp\dist_1).
    \end{equation}
\end{restatable}
Compared with the one in Theorem~\ref{thm:exact}, the lower bound of the joint error in Corollary~\ref{cor:joint} additionally depends on the zero-one entropy of $A$. In particular, if the marginal distribution of $A$ is skewed, then $\Hzo(A)$ will be small, which means that fairness will not reduce the joint accuracy too much. It is instructive to compare the above lower bound for the population error with the one of \citep[Theorem 2.3]{chzhen2020fair}, where the authors use a Wasserstein barycenter characterization to give the lower bound on the special case of $\ell_2$ error $(p=2)$ when the regressor can explicitly take the protected attribute as its input. As a comparison, our results apply to the general $\ell_p$ loss. In the special case of binary sensitive attribute with $\ell_2$ error, we have the following proposition hold:
\begin{restatable}{proposition}{comparison}(Informal)
\label{prop:comparison}
    Under the noiseless setting, the price (reduction of accuracy) paid by a fair regressor with explicit access to $A$ is higher than that of a fair regressor without access to $A$. 
\end{restatable}
Intuitively, because of the additional access to $A$, the optimal accuracy of an unfair regressor with access to $A$ is higher than that of an unfair regressor without access to $A$. However, under the fairness constraint, these two search spaces become the same, hence the price paid by the former is higher than the latter. 

From Corollary~\ref{cor:joint}, we can see that if either $\alpha\to 0$ or $\alpha\to 1$, i.e., the two subgroups are imbalanced in population, in this case even if $W_p(Y_\sharp\dist_0, Y_\sharp\dist_1)$ is large, it might seem like that the joint error $\eps_{p,\dist}(\Ypred)$ need not be large. However, this is due to the fact that the price in terms of the drop in accuracy is paid by the minority group. Our observation here suggests that the joint error $\eps_{p,\dist}(\Ypred)$ is not necessarily the objective to look at in high-stakes applications, since it naturally encodes the imbalance between different subgroups into account. Instead, a more appealing alternative to consider is the \emph{balanced error rate}:
\begin{equation}
    \text{Balanced Error Rate of }\Ypred\defeq\frac{1}{2}\left(\eps_{p,\dist_0}(\Ypred) + \eps_{p,\dist_1}(\Ypred)\right),
    \label{equ:ber}
\end{equation}
which applies balanced weights to both groups in the objective function. Clearly,~\eqref{equ:ber} could be reduced to the so-called \emph{cost-sensitive loss}, where data from group $a\in\{0, 1\}$ is multiplied by a positive weight that is reciprocal to the group's population level, i.e., $1 / \Pr(A = a)$. 

\subsection{Extension to Approximate Fairness under Noisy Setting}
In the last section we show that there is an inherent tradeoff between statistical parity and accuracy when a predictor \emph{exactly} satisfies statistical parity, and in particular this holds even if there is a perfect (unfair) regressor in both groups, \ie, there is no noise in the underlying population distribution. However, as formulated in~\eqref{equ:opt}, in practice we often only ask for approximate fairness where the quality of approximation is controlled by the slack variable $\epsilon$. Furthermore, even without the fairness constraint, in most interesting problems we often cannot hope to find perfect predictors for the regression problem of interest. Hence, it is natural to ask what is the tradeoff between fairness and accuracy when our predictor only approximately satisfies fairness ($\epsilon$-SP, Definition~\ref{def:esp}) over general distribution $\dist$? 

In this section we shall answer this question by generalizing our previous results to prove lower bounds on both the sum of conditional and the joint target errors that also take the quality of such approximation into account. Due to potential noise in the underlying distribution, we first define the excess risk $r_{p,\dist}(\Ypred)$ of a predictor $\Ypred$, which corresponds to the reducible error:
\begin{definition}[Excess Risk]
    Let $\Ypred = h(X)\in\RR$ be a predictor. The $\ell_p$ excess risk of $\Ypred$ is defined as $r_{p,\dist}(\Ypred)\defeq \eps_{p,\dist}(\Ypred) - \eps_{p,\dist}^*$, where $\eps_{p,\dist}^*\defeq \inf_f \eps_{p,\dist}(f(X))$ is the optimal error over all measurable functions.  
\end{definition}
Assume the infimum is achievable, we use $f_i^*$ to denote the optimal regressor without fairness constraint over $\dist_i, i\in\{0, 1\}$, \ie, $f_i^*\defeq \argmin_f \eps_{p,\dist_i}(f(X))$. Then we have the following hold:
\begin{restatable}{proposition}{generallowerbound}
    \label{prop:general}
    Let $\Ypred = h(X)$ be a predictor. For $p \geq 1$, if there exists $\epsilon > 0$ such that $W_p(h_\sharp\dist_0, h_\sharp\dist_1)\leq\epsilon$, then
    \begin{equation}
        r_{p,\dist_0}(\Ypred) + r_{p,\dist_1}(\Ypred) \geq \underbrace{W_p({f_0^*}_\sharp\dist_0, {f_1^*}_\sharp\dist_1)}_{\text{distance between optimal unfair predictors across groups}} - 2(\eps_{p,\dist_0}^* + \eps_{p,\dist_1}^*) - \epsilon,
        \label{equ:generallb}
    \end{equation}
    and $\Ypred$ satisfies $2\sqrt{C\epsilon}$-SP. 
\end{restatable}
\paragraph{Remark}
It is easy to verify that Proposition~\ref{prop:general} is a generalization of the lower bound presented in Theorem~\ref{thm:exact}: when $f_i^*$ are perfect predictors, we have $Y_\sharp\dist_i = {f_i^*}_\sharp\dist_i$ and $\eps_{p,\dist_i}^* = 0$, for $i\in\{0, 1\}$. Hence in this case the excess risk $r_{p,\dist_i}(\Ypred)$ reduces to the error $\eps_{p,\dist_i}(\Ypred)$. Furthermore, if $\epsilon = 0$, \ie, $\Ypred$ satisfies the exact statistical parity condition, then the lower bound~\eqref{equ:generallb} recovers the lower bound~\eqref{equ:lowerbound}. As a separate note, Proposition~\ref{prop:general} also implies that one can use the Wasserstein distance between the predicted distributions across groups as a proxy to ensure approximate statistical parity. This observation has also been shown in~\citet[Theorem 3.3]{dwork2012fairness} in classification.

\subsection{Individual Fairness, Accuracy Parity and the Wasserstein Distance}
In the last section we show that the Wasserstein distance between the output distributions across groups could be used as a proxy to ensure approximate statistical parity. Nevertheless, Theorem~\ref{thm:exact} and Proposition~\ref{prop:general} show that statistical parity is often at odds with the accuracy of the predictor, and in many real-world scenarios SP is insufficient to be used as a notion of fairness~\citep[Section 3.1]{dwork2012fairness}. Alternatively, in the literature a separate notion of fairness, known as \emph{individual fairness}, has been proposed in~\citet{dwork2012fairness}. Roughly speaking, a predictor $h$ is said to be individually fair if it treats similar individuals similarly:
\begin{definition}[Individual Fairness,~\citep{dwork2012fairness}]
\label{def:if}
    A regressor $h$ satisfies $\rho$-individual fairness if $\forall x,x'\in \xxspace$, $|h(x) - h(x')|\leq \rho \|x - x'\|$.
\end{definition}
Essentially, individual fairness puts a Lipschitz continuity constraint on the predictor. Note that in the original definition~\citep[Definition 2.1]{dwork2012fairness} the authors use a general metric $d_\xxspace(\cdot,\cdot)$ as a similarity measure between individuals, and the choice of such similarity measure is at the center of related applications. In this section we use $\|\cdot\|$ in Definition~\ref{def:if} mainly for the purpose of illustration, but the following results can be straightforwardly extended for any metric $d_\xxspace(\cdot,\cdot)$. Another notion of group fairness that has gained increasing attention~\citep{buolamwini2018gender,bagdasaryan2019differential,chi2021understanding} is \emph{accuracy parity}:
\begin{definition}[$\epsilon$-Accuracy Parity]
    \label{def:esp}
    Given a joint distribution $\dist$ and $0 \leq \epsilon \leq 1$, a regressor $\Ypred = h(X)$ satisfies $\epsilon$-\emph{accuracy parity} if $|\eps_{1,\dist_0}(\Ypred) - \eps_{1,\dist_1}(\Ypred)|\leq\epsilon$.
\end{definition}
Accuracy parity calls for approximately equalized performance of the predictor across different groups. The following proposition states the relationship between individual fairness, accuracy parity and the $W_1$ distance between the distributions $\dist_0$ and $\dist_1$ of different groups:
\begin{restatable}{proposition}{individual}
\label{prop:individual}
    If $h(\cdot)$ is $\rho$-individually fair, then $h$ satisfies $\sqrt{\rho^2 + 1}\cdot W_1(\dist_0, \dist_1)$-accuracy parity.
\end{restatable}
Proposition~\ref{prop:individual} suggests that in order to achieve approximate accuracy parity, one can constrain the predictor to be Lipschitz continuous while at the same time try to decrease the $W_1$ distance between the distributions across groups, via learning representations. In the case where the groups are similar and the Wasserstein distance is small, individual fairness provides some guidance towards approximate accuracy parity. However, in cases where the groups are different (disjoint), the representation learning becomes more important. 

\subsection{Fair Representations with Wasserstein Distance}
\label{sec:method}
Proposition~\ref{prop:general} and Proposition~\ref{prop:individual} suggest that the Wasserstein distance between the predicted distributions and the input distributions play a key role in controlling both statistical parity and accuracy parity, respectively. \emph{Is there a way to simultaneously achieve both goals?} In this section we shall provide an affirmative answer to this question via learning fair representations. The high-level idea is quite simple and intuitive: given input variable $X$, we seek to learn a representation $Z = g(X)$ such that $W_1(g_\sharp\dist_0, g_\sharp\dist_1)$ is small. If furthermore the predictor $h$ acting on the representation $Z$ is individually fair, we can then hope to have small statistical and accuracy disparity simultaneously.

Concretely, the following proposition says if the Wasserstein distance between feature distributions from two groups, $W_1(g_\sharp\dist_0, g_\sharp\dist_1)$, is small, then as long as the predictor is individually fair, it also satisfies approximate statistical parity:
\begin{restatable}{proposition}{wlip}
    \label{prop:wlip}
    Let $Z = g(X)$ be the features from input $X$. If $W_1(g_\sharp\dist_0, g_\sharp\dist_1) \leq \epsilon$ and $\Ypred = h(Z)$ is $\rho$-Lipschitz, then $\Ypred = (h\circ g)(X)$ verifies $2\sqrt{C\rho\epsilon}$-SP.
\end{restatable}
In practice since we only have finite samples from the corresponding distributions, we will replace all the distributions with their corresponding empirical versions. Furthermore, instead of using the joint error as our objective function, as we discussed in the last section, we propose to use the balanced error rate instead: 
\begin{equation}
    \begin{aligned}
         & \min_{g,~h}\max_{\|f\|_L\leq 1} &  & \frac{1}{2}\left(\eps_{2,\dist_0}(h\circ g) + \eps_{2,\dist_1}(h\circ g)\right) + \tau\cdot \left|\Exp_{g_\sharp\dist_0}[f(Z)] - \Exp_{g_\sharp\dist_1}[f(Z)]\right|,
    \end{aligned}
\label{equ:optp}
\end{equation}
where $\tau > 0$ is a hyperparameter that trades off the $\ell_p$ error and the Wasserstein distance. The above problem could be optimized using the gradient descent-ascent algorithm~\citep{edwards2015censoring,zhang2018mitigating}. To implement the Lipschitz constraint on the Wasserstein distance, we apply weight clipping to the parameters of both the adversary as well as the target predictor.

\section{Experiments}
\label{sec:experiment}
Our theoretical results imply that even if there is no significant drop in terms of the overall population error when a model is built to satisfy the statistical parity when the two demographic groups are imbalanced, the minority group can still suffer greatly from the reduction in accuracy. On the other hand, by using the balanced error rate as the objective function, we can mitigate the disparate drops in terms of accuracy between these two groups. Furthermore, by minimizing the Wasserstein distance of the feature distributions across groups, we can hope to simultaneously achieve approximate statistical parity and accuracy parity with further constraint on the predictor. To verify these implications, we conduct experiments on a real-world benchmark dataset, the Law School dataset~\citep{wightman1998lsac}, to present empirical results with various metrics. We refer readers to Appendix~\ref{sec:more} for further details about the Law School dataset, our pre-processing pipeline and the models used in the experiments.

\paragraph{Experimental Setup}
To demonstrate the effect of using Wasserstein distance to regularize the representations with adversarial training, we perform a controlled experiment by fixing the baseline model to be a three hidden-layer feed-forward network with ReLU activations, denoted as MLP. We use W-MLP to denote the model with Wasserstein constraint for representation learning. In the experiment, all the other factors are fixed to be the same across these two methods, including learning rate, optimization algorithm, training epoch, and also batch size. To see how the Wasserstein regularization affects the joint error, the conditional errors as well as the statistical parity and accuracy parity, we vary the coefficient $\tau$ for the adversarial loss between 0.1, 1.0, 5.0 and 10.0. For each experiment, we repeat each experiment for 5 times and report both the mean and the error bars.

\paragraph{Results and Analysis}
The experimental results are listed in Table~\ref{tab:results}. Note that in the table we use $K(\Ypred_0, \Ypred_1)$ to denote the Kolmogorov-Smirnov distance of the predicted distribution across groups, which is also the value of approximate statistical parity. From the table, it is then clear that with increasing $\tau$, both the statistical disparity and the accuracy disparity are decreasing. Interestingly, the overall error $\eps_\dist$ (sensitive to the marginal distribution of $A$) and the sum of group errors $\eps_{\dist_0} + \eps_{\dist_1}$ (insensitive to the imbalance of $A$) only marginally increase. In fact, for $\tau = 1.0$, we actually observed better accuracy. We conjecture that the improved performance stems from the implicit regularization via weight clipping of the target predictor. With $\tau = 10.0$, the last row shows that this method could effectively reduce both the statistical disparity and accuracy disparity to a value very close to 0, although at the cost of increasing errors. To conclude, all the empirical results are consistent with our findings.

\begin{table}[tb]
\small
    \centering
    \caption{Fair representations with Wasserstein regularization on the Law School dataset. We report the overall error, group-wise error, statistical disparity, and accuracy disparity.}
    \label{tab:results}
    \begin{tabular}{lScccc}\toprule
                &              $\tau$ & $\eps_\dist$ & $\eps_{\dist_0} + \eps_{\dist_1}$ & $K(\Ypred_{0}, \Ypred_1)$ & $|\eps_{\dist_0}(\Ypred) - \eps_{\dist_1}(\Ypred)|$\\\midrule
        MLP         &          N/A &  $0.034_{\pm 0.005}$  & $0.069_{\pm 0.011}$ & $0.296_{\pm 0.044}$ & $0.011_{\pm 0.004}$ \\
        W-MLP       &          0.1   &  $0.034_{\pm 0.004}$     & $0.067_{\pm 0.008}$ & $0.222_{\pm 0.037}$ & $0.011_{\pm 0.003}$  \\
        W-MLP       & 1.0   & $0.030_{\pm 0.000}$ & $0.059_{\pm 0.001}$ &  $0.116_{\pm 0.032}$ & $0.011_{\pm 0.002}$ \\
        W-MLP  & 5.0  &   $0.034_{\pm 0.002}$   & $0.067_{\pm 0.004}$ & $0.084_{\pm 0.073}$ & $0.008_{\pm 0.002}$ \\
        W-MLP  & 10.0 &   $0.035_{\pm 0.001}$ & $0.069_{\pm 0.003}$ & $0.048_{\pm 0.059}$ & $0.006_{\pm 0.000}$ \\
        \bottomrule
    \end{tabular}
\end{table}

\section{Related Work}
\label{sec:related}
\paragraph{Fair Regression}
Two central notions of fairness have been extensively studied, i.e., individual fairness and group fairness. In a seminal work, \citet{dwork2012fairness} defined individual fairness as a Lipschitz constraint of the underlying (randomized) algorithm. However, this definition requires apriori a distance metric to compute the similarity between pairs of individuals, which is often hard to construct or design in practice. Group fairness is a statistical definition, and it includes a family of refined definitions which essentially ask some statistical scores to be equalized between different subgroups. Typical examples include statistical parity (also known as demographic parity), equalized odds~\citep{hardt2016equality}, and accuracy parity~\citep{buolamwini2018gender}. In this work we focus on an extension of statistical parity to regression problems, and study its theoretical tradeoff with accuracy, when the regressor cannot directly take the protected attribute as input during both training and inference stages. We also investigate the relationship between individual fairness and accuracy parity, and provide a bound through the Wasserstein distance. The line of work on fair regression through regularization techniques dates at least back to~\citet{calders2013controlling}, where the authors enforce a first-order moment requirement between the predicted distributions. In a recent work,~\citet{agarwal2019fair} proposed a reduction approach from fair regression to a sequence of cost-sensitive minimization problems. Our definition of statistical parity in the regression setting is stronger than the one of~\citet{calders2013controlling}, which proposed to use the mean difference, i.e., the difference between the first-order moments of the group distributions, as the metric. Our definition also coincides with the one proposed by~\citet{agarwal2019fair}, which amounts to the Kolmogorov-Smirnov distance, when the output dimension is 1. 

\paragraph{Tradeoff between Fairness and Accuracy}
Although it has long been empirically observed that there is an inherent tradeoff between accuracy and statistical parity in both classification and regression problems~\citep{calders2009building,zafar2015fairness,zliobaite2015relation,berk2017convex,corbett2017algorithmic,zhao2020conditional}, precise characterizations on such tradeoffs are less explored. \citet[Proposition 8]{menon2018cost} explored such tradeoff in terms of the fairness frontier function under the context of cost-sensitive binary classification. \citet{zhao2019inherent} proved a lower bound on the joint error that has to be incurred by any fair algorithm satisfying statistical parity. Our negative result is similar to that of~\citet{zhao2019inherent} in nature, and could be understood as a generalization of their results from classification to regression. Recently, \citet{chzhen2020fair} and \citet{le2020projection} concurrently derived an analytic bound to characterize the price of statistical parity in regression when the learner can take the sensitive attribute explicitly as an input for $\ell_2$ loss. In this case, the lower bound is given by the optimal transportation distance from two group distributions to a common one, characterized by the $W_2$ barycenter. Our results differ in that in our case the learner cannot use the sensitive attribute as an input, and our results hold for the general $\ell_p$ loss. Note that this is significant, because it is not clear how to extend the results to space with $W_p$ as a metric, since the proof depends on the use of the Pythagoras' decomposition, which only holds under the $L_2$ distance. 

On the upside, under certain data generative assumptions of the sampling bias, there is a line of recent works showing that fairness constraints could instead improve the accuracy of the predictor~\citep{dutta2020there,blum2020recovering}. In particular, \citet{blum2020recovering} prove that if the observable data are subject to labeling bias, then the Equality of Opportunity constraint could help recover the Bayes optimal classifier. Note that this does not contradict with our results, since in this work we do not make any assumptions on the underlying training distributions.

\section{Conclusion}
\label{sec:conclusion}
In this paper we show that when the target distribution differs across different demographic subgroups, any fair algorithm in the statistical parity sense has to achieve a large error on at least one of the groups. In particular, we give a characterization of such tradeoff using the difference of the first order statistics (mean) of the target distributions from different groups. On the other hand, we also establish a connection between individual fairness and accuracy parity, where again, the accuracy disparity gap is characterized by the Wasserstein distance. Besides the theoretical contributions, our analysis using Wasserstein distance also suggests a practical algorithm for fair regression through learning representations for different demographic subgroups that are close in the sense of Wasserstein distance. Empirical results on a real-world dataset also confirm our findings.


\bibliography{reference}
\bibliographystyle{plainnat}

\newpage
\appendix
\section{Missing Proofs}
\label{sec:proof}
In this section we provide all the missing proofs in the main text. For the ease of the readers, in what follows we shall first restate the theorems that appear in the main text and then provide the corresponding proofs.
\subsection{Proofs of Theorem~\ref{thm:exact} and Corollary~\ref{cor:exact}}
\exactlowerbound*
\begin{proof}
    First, realize that $W_p(\cdot,\cdot)$ is a metric of probability distributions, the following chain of triangle inequalities holds:
    \begin{align*}
        W_p(Y_\sharp\dist_0, Y_\sharp\dist_1) \leq&~ W_p(Y_\sharp\dist_0, h_\sharp\dist_0) + W_p(h_\sharp\dist_0, h_\sharp\dist_1) + W_p(h_\sharp\dist_1, Y_\sharp\dist_1).
    \end{align*}
    Now due to the assumption that $\Ypred = h(X)$ is independent of $A$, the second term, $W_p(h_\sharp\dist_0, h_\sharp\dist_1)$, is 0, leading to:
    \begin{equation}
        W_p(Y_\sharp\dist_0, Y_\sharp\dist_1) \leq W_p(Y_\sharp\dist_0, h_\sharp\dist_0) + W_p(h_\sharp\dist_1, Y_\sharp\dist_1).
    \label{equ:chain}
    \end{equation}
    Next, for $a\in\{0, 1\}$, by definition of the Wasserstein distance,
    \begin{align}
        W_p(Y_\sharp\dist_a, h_\sharp\dist_a) =&~ \left(\inf_{\gamma}\Exp_{\gamma}[|Y - \Ypred|^p]\right)^{1/p} 
        \leq \left(\Exp_{\dist_a}[|Y - \Ypred|^p]\right)^{1/p} = \eps_{p,\dist_a}(\Ypred),
    \label{equ:error}
    \end{align}
    where we use the fact that the pushforward distribution of $\dist_a$ under $h$ is a particular coupling between $Y$ and $\Ypred$ to establish the above inequality. Applying the inequality~\eqref{equ:error} for both $a = 0$ and $a = 1$ and combining it with inequality~\eqref{equ:chain} completes the proof.
\end{proof}

\corexact*
\begin{proof}
    We first prove the first inequality in Corollary~\ref{cor:exact}. Apply Theorem~\ref{thm:exact} by setting $p = 1$. Let $\text{Id}:\yyspace\to\yyspace$ be the identity map, i.e., $\text{Id}(y) = y,~\forall y\in\RR$. Clearly $\text{Id}(\cdot)$ is 1-Lipschitz. Using the $\sup$ characterization of the Wasserstein distance, we have:
    \begin{align}
        \Exp_{\dist_0}[|\Ypred - Y|] + \Exp_{\dist_1}[|\Ypred - Y|] &\geq W_1(Y_\sharp\dist_0, Y_\sharp\dist_1) \tag{Theorem~\ref{thm:exact}} \nonumber\\ 
        &= \sup_{\|f\|_L\leq 1} \left|\int f~d(Y_\sharp\dist_0) - \int f~d(Y_\sharp\dist_1)\right| \nonumber\\
        &\geq \left|\int \text{Id}~d(Y_\sharp\dist_0) - \int \text{Id}~d(Y_\sharp\dist_1)\right| \tag{Id is 1-Lipschitz} \nonumber \\
        &= \left|\int Y~d\dist_0 - \int Y~d\dist_1\right| \tag{Change of Variable}\nonumber\\
        &= |\Exp_{\dist_0}[Y] - \Exp_{\dist_1}[Y]|,\label{equ:use}
    \end{align}
where the second to last equation follows from the definition of pushforward distribution.

To prove the second inequality in Corollary~\ref{cor:exact}, again set $p = 2$ in Theorem~\ref{thm:exact} and realize that for $a, b$, we have $a^2 + b^2 \geq 2(a + b)^2$ by the AM-GM inequality. Hence when $p = 2$, we have
    \begin{align*}
        \Exp_{\dist_0}[|\Ypred - Y|^2] + \Exp_{\dist_1}[|\Ypred - Y|^2] &\geq \Exp^2_{\dist_0}[|\Ypred - Y|] + \Exp^2_{\dist_1}[|\Ypred - Y|] \tag{Jensen's inequality}\\
        &\geq 2\left(\frac{\Exp_{\dist_0}[|\Ypred - Y|] + \Exp_{\dist_1}[|\Ypred - Y|]}{2}\right)^2 \tag{AM-GM inequality}\\
        &= \frac{1}{2}\left(\Exp_{\dist_0}[|\Ypred - Y|] + \Exp_{\dist_1}[|\Ypred - Y|]\right)^2 \\
        &\geq \frac{1}{2}|\Exp_{\dist_0}[Y] - \Exp_{\dist_1}[Y]|^2.\tag{Eq.~\eqref{equ:use}}
    \end{align*}
    Note that for the last equation we apply the lower bound $\Exp_{\dist_0}[|\Ypred - Y|] + \Exp_{\dist_1}[|\Ypred - Y|] \geq |\Exp_{\dist_0}[Y] - \Exp_{\dist_1}[Y]|$ we proved in the first inequality.
\end{proof}

\subsection{Proof of Theorem~\ref{thm:finite}}
Before we provide the proof of Theorem~\ref{thm:finite}, we first recall some useful results about the Wasserstein distance~\citep{weed2019sharp,lei2020convergence}.
\begin{proposition}[Proposition 20,~\citep{weed2019sharp}]
\label{prop:wconcentration}
For all $n\geq 0$ and $p \geq 1$, let $\hat{\dist}$ be an empirical distribution induced from $\dist$ with sample size $n$. Then, 
\begin{equation}
    \Pr(W_p^p(\dist, \hat{\dist}) \geq \Exp W_p^p(\dist, \hat{\dist}_n) + t) \leq \exp(-2nt^2).
\label{equ:wconcentration}
\end{equation}
\end{proposition}
Proposition~\ref{prop:wconcentration} gives a concentration inequality of $W_p^p(\cdot, \cdot)$ around its mean. Note that the expectation in~\eqref{equ:wconcentration} is over the draw of the sample of size $n$. This inequality is particularly useful when $p = 1$ since it reduces to $W_1(\cdot,\cdot)$ and gives a convergence rate of $O(1/\sqrt{n})$.


The following theorem is a special case of~\citep[Theorem 3.1]{lei2020convergence}, which bounds the rate of $\Exp W_p(\dist, \hat{\dist}_n)$:
\begin{theorem}[Theorem 3.1,~\citep{lei2020convergence}]
\label{thm:wmean}
    Let $\hat{\dist}$ be an empirical distribution induced from $\dist$ with sample size $n$, $p\geq 1$ and recall that $\yyspace = [-1, 1]$. Then
    \begin{equation}
        \Exp W_p(Y_\sharp\dist, Y_\sharp\hat{\dist}) \leq c_p \cdot n^{-\frac{1}{2p}},
    \label{equ:wmean}
    \end{equation}
    where $c_p$ is a positive constant that only depends on $p$.
\end{theorem}
Again, the interesting case here is when $p = 1$, which gives the same rate of $O(1/\sqrt{n})$ that coincides with the one in Proposition~\ref{prop:wconcentration}.

\finite*
\begin{proof}
    We first prove the finite sample lower bound w.r.t.\ the $\ell_1$ error. Realize that $W_1(\cdot,\cdot)$ is a metric, the triangle inequality gives us
    \begin{equation*}
        W_1(Y_\sharp\hat{\dist}_0, Y_\sharp\hat{\dist}_1) \leq W_1(Y_\sharp\hat{\dist}_0, Y_\sharp\dist_0) + W_1(Y_\sharp\dist_0, Y_\sharp\dist_1) + W_1(Y_\sharp\dist_1, Y_\sharp\hat{\dist}_1).
    \end{equation*}
    Combined with Theorem~\ref{thm:exact}, the above inequality leads to 
    \begin{equation*}
        \eps_{1,\dist_0}(\Ypred) + \eps_{1,\dist_1}(\Ypred) \geq W_1(Y_\sharp\hat{\dist}_0, Y_\sharp\hat{\dist}_1) - \left(W_1(Y_\sharp\hat{\dist}_0, Y_\sharp\dist_0) + W_1(Y_\sharp\dist_1, Y_\sharp\hat{\dist}_1)\right).
    \end{equation*}
    Hence it suffices if we could provide high probability bound to further lower bound $W_1(Y_\sharp\dist_i, Y_\sharp\hat{\dist}_i)$, for $i\in\{0, 1\}$. To this end, we first apply Proposition~\ref{prop:wconcentration} with $p = 1$: let $\exp(-2nt^2) = \delta / 2$ and solve for $t$, we have $t = \sqrt{\log(2/\delta)/ 2n}$, which means that with probability at least $1 - \delta / 2$, 
    \begin{align*}
        W_1(Y_\sharp\hat{\dist}_i, Y_\sharp\dist_i) &\leq \Exp W_1(Y_\sharp\hat{\dist}_i, Y_\sharp\dist_i) + \sqrt{\frac{\log(2/\delta)}{2n}}  \\
        &\leq c_p\sqrt{\frac{1}{n}} + \sqrt{\frac{\log(2/\delta)}{2n}} \tag{Theorem~\ref{thm:wmean}}.
    \end{align*}
    Now apply the above inequality twice, one for $i\in\{0, 1\}$. With a union bound, we have shown that w.p. $\geq 1-\delta$, 
    \begin{equation*}
        \eps_{1,\dist_0}(\Ypred) + \eps_{1,\dist_1}(\Ypred) \geq W_1(Y_\sharp\hat{\dist}_0, Y_\sharp\hat{\dist}_1) - \left(2c_1 + \sqrt{2\log(2/\delta)}\right)\sqrt{\frac{1}{n}}.
    \end{equation*}
    To prove the second lower bound w.r.t.\ the $\ell_2$ error, simply realize that $\eps_{2,\dist_i}(\Ypred) \geq \eps_{1,\dist_i}(\Ypred)$ for $i\in\{0, 1\}$, which completes the proof.
\end{proof}

\subsection{Proof of Corollary~\ref{cor:joint}}
\jointlowerbound*
\begin{proof}
    To simplify the notation used in the proof, define $\eps\defeq\eps_{p,\dist}(\Ypred)$, $\eps_0\defeq\eps_{p, \dist_0}(\Ypred)$ and $\eps_1\defeq\eps_{p,\dist_1}(\Ypred)$. Let $\alpha\defeq\Pr_\dist(A = 0)$. By Theorem~\ref{thm:exact}, we know that $\eps_0 + \eps_1 \geq W_p(Y_\sharp\dist_0, Y_\sharp\dist_1)$. By definition of the joint error:
    \begin{align*}
        \eps &= \alpha\eps_0 + (1-\alpha)\eps_1 \geq \alpha\eps_0 + (1-\alpha)(W_p(Y_\sharp\dist_0, Y_\sharp\dist_1) - \eps_0) \\
        &= (1-\alpha)W_p(Y_\sharp\dist_0, Y_\sharp\dist_1) + (2\alpha - 1)\eps_0.
    \end{align*}
    Similarly, we can also lower bound the joint error by:
    \begin{equation*}
        \eps \geq \alpha W_p(Y_\sharp\dist_0, Y_\sharp\dist_1) + (1-2\alpha)\eps_1.
    \end{equation*}
    Now we discuss in two cases. If $\alpha\leq 1/2$, considering the second inequality yields:
    \begin{align*}
        \eps \geq \alpha W_p(Y_\sharp\dist_0, Y_\sharp\dist_1) + (1-2\alpha)\eps_1 \geq \alpha W_p(Y_\sharp\dist_0, Y_\sharp\dist_1).
    \end{align*}    
    If $\alpha > 1/2$, using the first inequality we have:
    \begin{align*}
        \eps \geq (1-\alpha)W_p(Y_\sharp\dist_0, Y_\sharp\dist_1) + (2\alpha - 1)\eps_0 \geq (1-\alpha)W_p(Y_\sharp\dist_0, Y_\sharp\dist_1).
    \end{align*}    
    Combining the above two cases leads to:
    \begin{equation*}
        \eps \geq \min\{\alpha, 1-\alpha\}\cdot W_p(Y_\sharp\dist_0, Y_\sharp\dist_1) = \Hzo(A)\cdot W_p(Y_\sharp\dist_0, Y_\sharp\dist_1),
    \end{equation*}
    completing the proof.
\end{proof}

\subsection{Proof of Proposition~\ref{prop:comparison}}
\comparison*
To prove this proposition, we first state the theorem in the setting where the regressor has explicit access to the protected attribute $A$ from~\citet{chzhen2020fair} (using adapted notation for consistency):
\begin{theorem}
Assume, for each $a\in\{0, 1\}$, that the univariate measure $Y_\sharp\mu_a$ has a density and let $p_a\defeq \Pr(A = a)$. Then,
\begin{equation*}
\min_{g\text{ is fair}}\quad \Exp[(Y - g(X, A))^2] = \min_{\nu}\sum_{a\in\{0, 1\}}p_a\cdot W_2^2(Y_\sharp\mu_a, \nu),
\end{equation*}
where $\nu$ is a measure over $\RR$.
\end{theorem}
Now we can proceed to prove the statement in Proposition~\ref{prop:comparison}.
\begin{proof}
Without loss of generality, let $p_0 \defeq \Pr(A = 0)\geq p_1\defeq \Pr(A = 1)$. For the special case of $W_2$ with $\|\cdot\|_2$ as the underlying metric, we know that the Wasserstein barycenter lies on the Wasserstein geodesic between $Y_\sharp\mu_0$ and $Y_\sharp\mu_1$~\citep{villani2009optimal}. Let $\nu^* = \argmin \sum_{a\in\{0, 1\}}p_a\cdot W_2^2(Y_\sharp\mu_a, \nu)$, i.e., $\nu^*$ is the Wasserstein barycenter. Now since $W_2(\cdot,\cdot)$ is a metric and $\nu^*$ lies on the geodesic, we know
\begin{equation}
\label{equ:geo}
    W_2(Y_\sharp\mu_0, \nu^*) + W_2(Y_\sharp\mu_1, \nu^*) = W_2(Y_\sharp\mu_0, Y_\sharp\mu_1).
\end{equation}
Compare the prices paid in these two cases:
\begin{align*}
    & \text{Price in $\ell_2$ when the regressor cannot access $A$}: && \Hzo(A)\cdot W_p(Y_\sharp\dist_0, Y_\sharp\dist_1), \\
    & \text{Price in $\ell_2$ when the regressor can access $A$}: && \sqrt{\sum_{a\in\{0, 1\}}p_a\cdot W_2^2(Y_\sharp\mu_a, \nu^*)}.
\end{align*}
It is easy to see that
\begin{align*}
    \sum_{a\in\{0, 1\}}p_a\cdot W_2^2(Y_\sharp\mu_a, \nu^*) &= p_0\cdot W_2^2(Y_\sharp\mu_0, \nu^*) + p_1\cdot W_2^2(Y_\sharp\mu_1, \nu^*) \\
                                                            &\geq p_1\cdot \left(W_2^2(Y_\sharp\mu_0, \nu^*) +  W_2^2(Y_\sharp\mu_1, \nu^*)\right) \tag{$p_0\geq p_1$}\\
                                                            &\geq \frac{p_1}{2}\left(W_2(Y_\sharp\mu_0, \nu^*) +  W_2(Y_\sharp\mu_1, \nu^*)\right)^2 \tag{AM-GM inequality} \\
                                                            &= \frac{p_1}{2}\cdot W_2^2(Y_\sharp\mu_0, Y_\sharp\mu_1) \tag{Eq.~\eqref{equ:geo}} \\
                                                            &\geq p_1^2\cdot W_2^2(Y_\sharp\mu_0, Y_\sharp\mu_1)\tag{$p_1\leq 1/2$} \\
                                                            &= \Hzo(A)^2 \cdot W_2^2(Y_\sharp\dist_0, Y_\sharp\dist_1),
\end{align*}            
completing the proof.
\end{proof}

\subsection{Proof of Proposition~\ref{prop:general}}
As a comparison to the Kolmogorov-Smirnov distance, the $W_1$ distance between distributions over $\RR$ could be equivalently represented as:
\begin{proposition}[\citet{gibbs2002choosing}]
    For two distributions $\dist$, $\dist'$ over $\RR$, $W_1(\dist,\dist') = \int_{\RR}|F_\dist(z) - F_{\dist'}(z)|~dz$.
    \label{prop:wasserstein}
\end{proposition}
Proposition~\ref{prop:wasserstein} was stated as a fact without proof in~\citep{gibbs2002choosing}, but it is not hard to see that it could be proved using the equivalent characterization of $W_1$ in~\eqref{equ:quantile} by changing the integral variable. Furthermore, in regression if both $\dist$ and $\dist'$ are continuous distributions, then the following well-known result serves as a bridge to connect the Wasserstein distance $W_1(\cdot, \cdot)$ and the Kolmogorov-Smirnov distance $K(\cdot,\cdot)$:
\begin{lemma}
    If there exists a constant $C$ such that the density of $\dist'$ (w.r.t.\ the Lebesgue measure $\lambda$) is universally bounded above, i.e., $\|d\dist' / d\lambda\|_\infty \leq C$, then $K(\dist, \dist') \leq 2\sqrt{C\cdot W_1(\dist, \dist')}$.
    \label{lemma:kol}
\end{lemma}
Using Kolmogorov-Smirnov distance, the constraint in the optimization problem~\eqref{equ:opt} could be equivalently expressed as $K(h_\sharp\dist_0, h_\sharp\dist_1) \leq \epsilon$. Now with Lemma~\ref{lemma:kol}, we are ready to prove Proposition~\ref{prop:general}:
\generallowerbound*
\begin{proof}
    First, for $a\in\{0, 1\}$, by definition of the Wasserstein distance, for any predictor $\Ypred = h(X)$:
    \begin{align}
        W_p(Y_\sharp\dist_a, h_\sharp\dist_a) =&~ \left(\inf_{\gamma}\Exp_{\gamma}[|Y - \Ypred|^p]\right)^{1/p} 
        \leq \left(\Exp_{\dist_a}[|Y - \Ypred|^p]\right)^{1/p} = \eps_{p,\dist_a}(\Ypred),
    \end{align}
    Applying the above inequality to both $h$ and $f^*_a$, we have:
    \begin{equation}
        W_p(h_\sharp\dist_a, Y_\sharp\dist_a) + W_p(Y_\sharp\dist_a, {f_a^*}_\sharp\dist_a) \leq \eps_{p,\dist_a}(\Ypred) + \eps_{p,\dist_a}(f_a^*(X)) = \eps_{p,\dist_a}(\Ypred) + \eps_{p,\dist_a}^*,\quad\forall a\in\{0, 1\}.
    \label{equ:bridge}
    \end{equation}
    On the other hand, by the triangle inequality,
    \begin{equation*}
        W_p(h_\sharp\dist_0, h_\sharp\dist_1) + \sum_{a\in\{0, 1\}} W_p(h_\sharp\dist_a, Y_\sharp\dist_a) + W_p(Y_\sharp\dist_a, {f_a^*}_\sharp\dist_a) \geq W_p({f_0^*}_\sharp\dist_0, {f_1^*}_\sharp\dist_1).
    \end{equation*}
    Now by the assumption $W_p(h_\sharp\dist_0, h_\sharp\dist_1) \leq \epsilon$ and Eq.~\eqref{equ:bridge}, we have:
    \begin{equation*}
        \epsilon + \sum_{a\in\{0, 1\}} \eps_{p,\dist_a}(\Ypred) + \eps_{p,\dist_a}^* \geq W_p({f_0^*}_\sharp\dist_0, {f_1^*}_\sharp\dist_1).
    \end{equation*}
    By the definition of the excess risk, rearranging and subtracting $2\sum_{a\in\{0, 1\}} \eps_{p,\dist_a}^*$ from both sides of the inequality then completes the proof of the first part. 
    
    To show that $\Ypred$ is $2\sqrt{C\epsilon}$-SP, first note that $\Ypred = h(X)$ is $t$-SP iff $K(h_\sharp\dist_0, h_\sharp\dist_1)\leq t$. Now apply Lemma~\ref{lemma:kol}, under the assumption that $W_p(h_\sharp\dist_0, h_\sharp\dist_1)\leq \epsilon$:
    \begin{align*}
        K(h_\sharp\dist_0, h_\sharp\dist_1) & \leq 2\sqrt{CW_1(h_\sharp\dist_0, h_\sharp\dist_1)} \tag{Lemma~\ref{lemma:kol}}\\
                                            & \leq 2\sqrt{CW_p(h_\sharp\dist_0, h_\sharp\dist_1)} \tag{Monotonicity of the $W_p(\cdot,\cdot)$}\\
                                            &\leq 2\sqrt{C\epsilon},
    \end{align*}
    completing the proof.
\end{proof}

\subsection{Proof of Proposition~\ref{prop:individual}}
\individual*
\begin{proof}
Define $g(X, Y)\defeq |\Ypred - Y| = |h(X) - Y|$. We first show that if $h(X)$ is $\rho$-Lipschitz, then $g(X, Y)$ is $\sqrt{\rho^2 + 1}$-Lipschitz: for $\forall x, y, x, x'$:
\begin{align*}
    |g(x, y) - g(x', y')| &= \left||h(x) - y| - |h(x') - y'|\right| \\
                          &\leq |h(x) - h(x') - y + y'| \tag{Triangle inequality}\\
                          &\leq |h(x) - h(x')| + |y - y'| \\
                          &\leq \rho \|x - x'\| + |y - y'| \tag{$h$ is $\rho$-Lipschitz}\\
                          &\leq \sqrt{\rho^2 + 1}\cdot \sqrt{\|x - x'\|^2 + |y - y'|^2} \tag{Cauchy-Schwarz} \\
                          &= \sqrt{\rho^2 + 1} \cdot \|(x,y) - (x',y')\|.
\end{align*}
Let $\rho'\defeq \sqrt{\rho^2 + 1}$. Now consider the error difference:
\begin{align*}
    |\eps_{1,\dist_0}(\Ypred) - \eps_{1,\dist_1}(\Ypred)| &= |\Exp_{\dist_0}[|h(X) - Y|] - \Exp_{\dist_1}[|h(X) - Y|]| \\
                                                          &= |\Exp_{\dist_0}[g(X, Y)] - \Exp_{\dist_1}[g(X, Y)]| \\
                                                          &\leq \sup_{\|g'\|_L\leq \rho'} |\Exp_{\dist_0}[g'(X, Y)] - \Exp_{\dist_1}[g'(X, Y)]| \\
                                                          &= \rho' \sup_{\|g'\|_L\leq 1} |\Exp_{\dist_0}[g'(X, Y)] - \Exp_{\dist_1}[g'(X, Y)]| \\
                                                          &= \rho'\cdot W_1(\dist_0, \dist_1) \tag{Kantorovich duality},
\end{align*}
which completes the proof.
\end{proof}
    
\subsection{Proof of Proposition~\ref{prop:wlip}}
\wlip*
\begin{proof}
We first show that $W_1((h\circ g)_\sharp\dist_0, (h\circ g)_\sharp\dist_1)$ is small if $h$ is $\rho$-Lipschitz. To simplify the notation, we define $\dist'_0 \defeq g_\sharp\dist_0$ and $\dist'_1 \defeq g_\sharp\dist_1$. Consider the dual representation of the Wasserstein distance:
    \begin{align*}
        W_1(h_\sharp\dist'_0, h_\sharp\dist'_1) =&~ \sup_{\|f'\|_L\leq 1}\left|\int f'~d(h_\sharp\dist'_0) - \int f'~d(h_\sharp\dist'_1)\right| \tag{Kantorovich duality}\\
        =&~ \sup_{\|f'\|_L\leq 1}\left|\int f'\circ h~d\dist'_0 - \int f'\circ h~d\dist'_1\right| \tag{Change of Variable formula}\\
        \leq&~ \sup_{\|f\|_L\leq \rho} \left|\int f~d\dist'_0 - \int f~d\dist'_1\right| \tag{$h$ is $\rho$-Lipschitz}\\
        =&~ \rho\cdot W_1(\dist'_0, \dist'_1) \\
        =&~ \rho\cdot W_1(g_\sharp\dist_0, g_\sharp\dist_1) \\
        \leq&~ \rho\epsilon,
    \end{align*}
where the first inequality is due to the fact that for $\|f'\|_L\leq 1$, $\|f'\circ h\|_L\leq \|f'\|_L\cdot \|h\|_L = \rho$. Applying Lemma~\ref{lemma:kol} to $W_1(h_\sharp\dist'_0, h_\sharp\dist'_1)$ then completes the proof.
\end{proof}

\section{Further Details about the Experiments}
\label{sec:more}
\subsection{Dataset}
The Law School dataset contains 1,823 records for law students who took the bar passage study for Law School Admission.\footnote{We use the edited public version of the dataset which can be downloaded here: \url{https://github.com/algowatchpenn/GerryFair/blob/master/dataset/lawschool.csv}} The features in the dataset include variables such as undergraduate GPA, LSAT score, full-time status, family income, gender, etc. In the experiment, we use gender as the protected attribute and undergraduate GPA as the target variable. We use 80 percent of the data as our training set and the rest 20 percent as the test set. The data distribution for different subgroups in the Law School dataset could be found in Figure~\ref{fig:data-dist-law}. In the Law School dataset, $\Pr(A = 1) = 0.452$, which is a quite balanced dataset. All the experiments are performed on a Titan 1080 GPU.
\begin{figure}[htb]
\centering
    \includegraphics[width=0.6\linewidth]{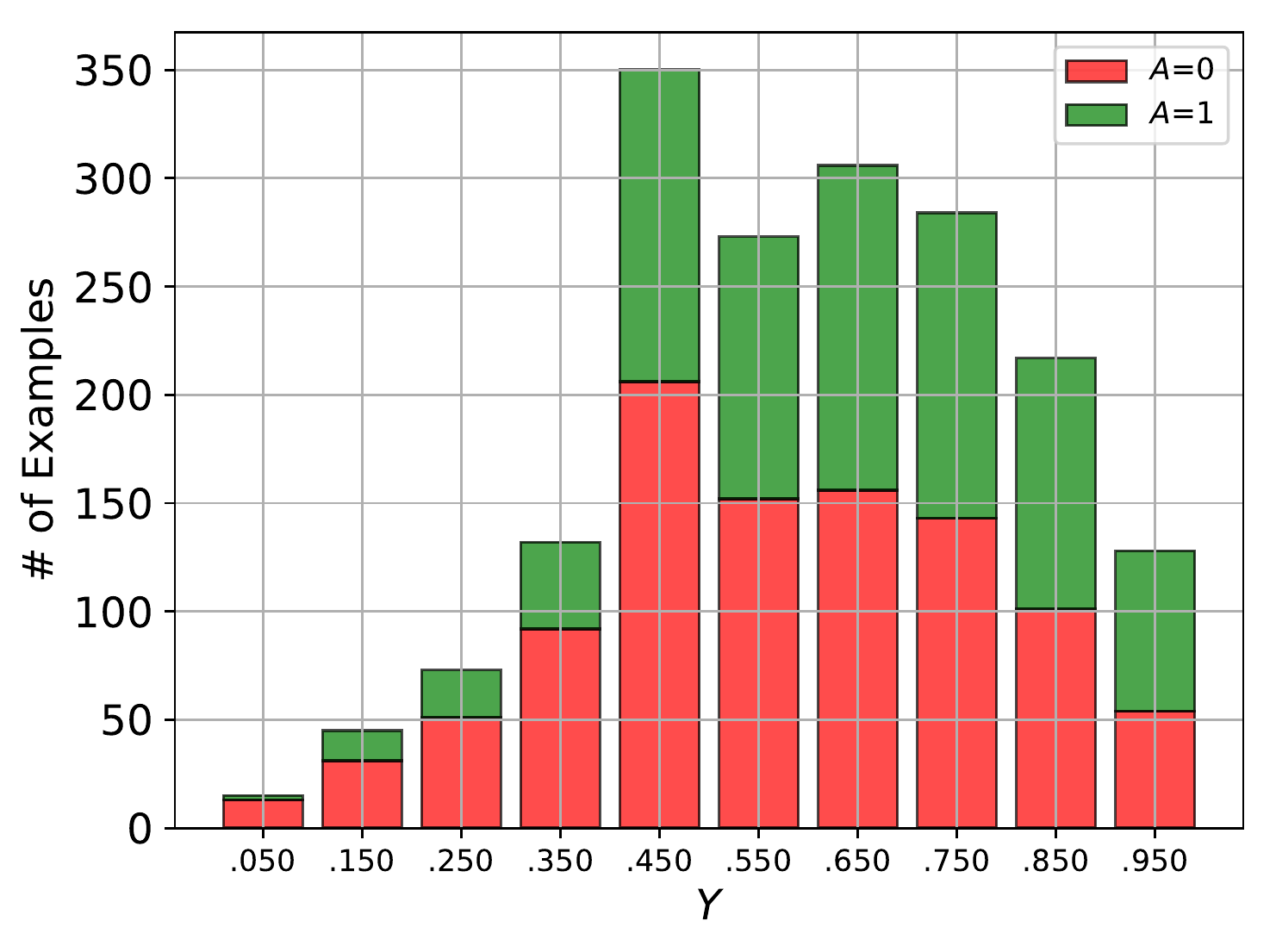}
    \caption{The data distributions of different groups in the Law School dataset.}
    \label{fig:data-dist-law}
\end{figure}

\subsection{Network Architectures}
We fix the baseline model to be a three hidden-layer feed-forward network with ReLU activations. The number of units in each hidden layer are 50 and 20, respectively. The output layer corresponds to a linear regression model. This baseline is denoted as MLP. For learning with Wasserstein regularization, the adversarial discriminator network takes the feature from the last hidden layer as input, and connects it to a hidden-layer with 10 units, followed by an auditor whose goal is to output a score function in distinguishing the features from the two different groups. This model is denoted as W-MLP. Compared with MLP, the only difference of W-MLP in terms of objective function is that besides the $\ell_2$ loss for target prediction, the W-MLP also contains a loss from the auditor to distinguish the sensitive attribute $A$. 

\subsection{Hyperparameters used in Experiments}
In this section we report the detailed hyperparameters used in our experiments to obtain the results in Table~\ref{tab:results}. Throughout the experiments, we fix the learning rate to be 1.0 and use the same networks as well as random seeds. One important aspect in the implementation of the Wasserstein adversary is the choice of the clipping parameter for the weights in the adversary network. The values used in our experiments are shown below in Table~\ref{tab:hyper}.
\begin{table}[htb]
    \centering
    \caption{Clipping parameters used in training the Wasserstein adversary.}
    \label{tab:hyper}
    \begin{tabular}{lSS}\toprule
                &   $\tau$ & Clipping Value\\\midrule
        W-MLP       &          0.1   &  0.1\\
        W-MLP       & 1.0   & 1.0\\
        W-MLP  & 5.0  &     5.0                                 \\
        W-MLP  & 10.0 & 10.0  \\
        \bottomrule
    \end{tabular}
\end{table}

\end{document}